\documentclass[twoside]{article}
\usepackage[accepted]{aistats2012}

\usepackage{epsfig}
\usepackage{subfigure}
\usepackage{calc}
\usepackage{amssymb}
\usepackage{amstext}
\usepackage{amsmath}
\usepackage{amsthm}
\usepackage{multicol}
\usepackage{pslatex}
\usepackage{natbib}
\bibpunct[; ]{(}{)}{,}{a}{}{;}
\makeatletter
\let\bibhang\@undefined
\makeatother
\usepackage{apalike}
\usepackage{algorithmic}
\usepackage{algorithm}
\usepackage{subfigure}
\usepackage{url}
\usepackage{multirow}

\newtheorem{theorem}{Theorem}

\newtheorem{proposition}{Proposition}
\newcommand{\mathds}[1]{\mathbb{#1}}
\newcommand{\beq}{\begin{equation}}
\newcommand{\eeq}{\end{equation}}
\newcommand{\IR}{\mathds{R}}
\newcommand{\E}{\mathds{E}}
\newcommand{\Exp}{\mathds{E}}

\newcommand{\cH}{{\mathcal H}}

\newcommand{\cS}{{\mathcal S}}

\def\sjwresolved#1{}

\def\noprint#1{}

\newcommand{\btab}{\begin{tabbing}
\ \ \= thenn \= thenn \= thenn \= thenn \= thenn \= \kill}
\newcommand{\etab}{\end{tabbing}}

\usepackage{xspace}

\newcommand{\svmlight}{SVM-Light\xspace}
\newcommand{\svmperf}{SVM-Perf\xspace}
\newcommand{\asset}{ASSET$_M$\xspace}
\newcommand{\assetst}{ASSET$^*_M$\xspace}
\newcommand{\assetonline}{ASSET$_F$\xspace}%{ASSET$_{on}$\xspace}
\newcommand{\assetstonline}{ASSET$^*_F$\xspace}
\newcommand{\assetboth}{ASSET$_M^{(*)}$\xspace}
\newcommand{\assetbothon}{ASSET$^{(*)}_{F}$\xspace}

\newcommand{\vs}[0]{\mathbf{s}}
\newcommand{\vt}[0]{\mathbf{t}}

\newcommand{\vw}[0]{\mathbf{w}}

\newcommand{\vzero}[0]{\mathbf{0}}

\newcommand{\tilphi}[0]{{\phi}}

\newcommand{\PsiG}[0]{\Psi^\circ} %\Psi_G
\newcommand{\phiG}[0]{\phi^\circ} %\phi_G
\newcommand{\kG}{k^\circ} %k^G

\newcommand{\sg}{\hspace{2mm}}
\newcommand{\zg}{\hspace{0.2mm}}
\newcommand{\mc}{\multicolumn}

\def\sjwresolved#1{}
\def\sklresolved#1{}

\def\noprint#1{}

\begin{document}

%\title{{ASSET: Approximate Stochastic Subgradient Estimation Training for Support Vector Machines}}
%
%\author{\authorname{Sangkyun Lee\sup{1} and Stephen J. Wright\sup{2}}
%\affiliation{\sup{1}Computer Science Department, LS VIII, University of Technology Dortmund, Dortmund, Germany}
%\affiliation{\sup{2}Department of Computer Sciences, University of Wisconsin-Madison, Madison WI, USA}
%\email{sangkyun.lee@tu-dortmund.de, swright@cs.wisc.edu}
%}
%\author{
%}

%\keywords{Stochastic Approximation : Large-Scale : Online Learning : Support Vector Machines : Nonlinear Kernels}

\runningtitle{Approximate Stochastic Subgradient Estimation Training for SVMs}

\twocolumn[
\aistatstitle{Approximate Stochastic Subgradient Estimation Training\\ for Support Vector Machines}

% unless you've provided the [accepted] option to the icml2011
\aistatsauthor{Sangkyun Lee  \And Stephen J. Wright }
\aistatsaddress{ Computer Science Department, LS VIII\\ University of
  Technology\\ Dortmund, Germany \\ \url{sangkyun.lee@uni-dortmund.de} \And
Computer Sciences Department\\ University of Wisconsin \\Madison, USA \\ \url{swright@cs.wisc.edu}}

%\icmlkeywords{stochastic approximation, large-scale, online learning, support vector machines, nonlinear kernels}

%\vskip 0.3in
]

%\abstract{The abstract should summarize the contents of the paper and should contain at least 70 and at most 200 words. The text must be set to 9-point font size.}
\begin{abstract}
  Subgradient algorithms for training support vector machines have
  been quite successful for solving large-scale and online learning
  problems. However, they have been restricted to linear kernels and
  strongly convex formulations. This paper describes efficient
  subgradient approaches without such limitations. Our approaches make
  use of randomized low-dimensional approximations to nonlinear
  kernels, and minimization of a reduced primal formulation using an
  algorithm based on robust stochastic approximation, which do not
  require strong convexity. Experiments illustrate that our approaches
  produce solutions of comparable prediction accuracy with the
  solutions acquired from existing SVM solvers, but often in much
  shorter time. We also suggest efficient prediction schemes that
  depend only on the dimension of kernel approximation, not on the
  number of support vectors.
\end{abstract}

%\onecolumn \maketitle \normalsize \vfill

\section{{Introduction}}\label{sec:intro}
\noindent
Support vector machines (SVMs) have been highly successful in machine
learning and data mining. Derivation, implementation, and analysis of
efficient solution methods for SVMs have been the subject of a great
deal of research during the past 12 years. We broadly categorize the
algorithms that have been proposed as follows.
\begin{itemize}
\item[(i)] {\em Decomposition methods} based on the dual SVM
  formulation, including SMO~\citep{Platt99}, LIBSVM~\citep{Fan05},
  {\svmlight}~\citep{Joa99}, GPDT~\citep{SerZZ05a}, and an online
  variant LASVM~\citep{BorAnt05}. The dual formulation allows
  nonlinear kernels to be introduced neatly into the formulation via
  the kernel trick~\citep{BosGI92}. %\cite{SchS02}.
  % , allowing more general and more powerful classifiers to be
  % derived than the planar classifiers obtained from linear kernels.
\item[(ii)] {\em Cutting-plane methods} using special primal
  formulations to successively add violated constraints to the
  formulation. {\svmperf}~\citep{Joa06} and OCAS~\citep{FraS08} handle
  linear kernels, while the former approach is extended to nonlinear
  kernels in CPNY~\citep{JoaFY09} and CPSP~\citep{JoaYu09}.
\item[(iii)] {\em Subgradient methods} for the primal
  formulations. Available codes include Pegasos~\citep{ShaSSS07} and
  SGD~\citep{Bot05a}. These require linear kernels and strong
  convexity of the SVM formulation.
\end{itemize}

\noindent
% When data sets are extremely large, the computation required by some
% of the existing algorithms becomes excessive.
Subgradient methods are of particular interest, since they are well
suited to large-scale and online learning problems.
% They take simple steps, each typically based on a single training
% point, so can be implemented in a data-streaming context. They are
% simple to implement.
Each iteration of these methods consists of simple computation,
usually involving a tiny subset of training data. Although a large
number of iterations might be required to find high accuracy
solutions, solutions of moderate accuracy are often enough for
learning purposes.
%While requiring many iterates to find accurate solutions, they
%can sometimes calculate solutions that are ``accurate enough'' for the
%purposes at hand using much less computation than frameworks that more
%explicitly target an exact solution. 
Despite such benefits, no subgradient algorithms have yet been
proposed for SVMs with {\em nonlinear kernels}, due mainly to the lack
of explicit representations for feature mappings of interesting
kernels, which are required in the primal
formulations. % are often not available.
% required in the primal formulations to use kernels.  difficulity of
% specifying feature mapping function explicitly in the primal
% formulation used by these methods, which is not available in
% general.
This paper aims to provide practical subgradient algorithms for
training SVMs with nonlinear kernels.

Unlike Pegasos~\citep{ShaSSS07}, we use Vapnik's original SVM
formulation without modifying the objective to be strongly convex.
% we do not modify the SVM formulation in order to have a strongly
% convex objective and thereby to apply $O(1/t)$ steplengths.
Our main algorithm takes steplengths of size $O(1/\sqrt{t})$
(associated with robust stochastic approximation
methods~\citep{NemJ09,NemY83} and online convex
programming~\citep{Zin03}), rather than the $O(1/t)$ steplength scheme
in Pegasos.
% and incremental subgradient
%methods~\cite{NedB00a} from the optimization literature, 
% Subgradient methods also seem appealing to the latest scientific
% discovery, where we have very large amount of data with possibly
% high measurement errors, and therefore strict optimality of
% solutions can be sacrificed on behalf of finding out reasonable
% solutions in a short amount of time for rapid modeling and
% statistical analysis.
Although the $O(1/\sqrt{t})$ schemes have slower convergence rate in
theory, we see no significant performance difference in practice to
$O(1/t)$ methods.
As we discuss later, optimal choices of a tuning
parameter in the objective often lead it to be nearly weakly convex,
thus nearly breaking the assumption that underlies the $O(1/t)$
scheme.
% This is because the degree of strongly convexity in the objective
% depends on a tuning parameter, for which the optimal choices often
% lead to very weak strong convexity. Therefore the condition required
% for $O(1/t)$ scheme often breaks in practice.  the value of the
% optimal tuning parameter in the objective often results in very weak
% strong convexity, and thus breaks the assumption required for the
% $O(1/t)$ scheme.  leading to no significant performance difference
% to the former.

%This paper outlines an improved algorithm based on subgradient methods
%for solving primal SVM formulations. It extends current subgradient
%methods by allowing nonlinear kernels to be used, and not requiring
%strict convexity of the function to be minimized. This allows the
%classic SVM formulation with a non-penalized intercept term to be
%used, thus reclaiming the formulation on which many theoretic results
%have been built.

In our nonlinear-kernel formulation, we use low-dimensional
approximations
%with explicit representations 
to the nonlinear feature mappings, whose dimension can be chosen by
users.
%For nonlinear
%kernels, we use low-dimensional approximations to the feature
%mappings, with explicit representation. 
% Our approach uses low-dimensional approximations to nonlinear
% kernels, 
We obtain such approximations either by approximating the Gram matrix
or by constructing subspaces with random bases approximating the
feature spaces induced by kernels.
% By considering the images of feature mappings as transformed input
% vectors, the SVMs with nonlinear kernels can be treated as linear
% SVMs The dimension of approximation can be controlled by users,
These approximations can be computed and applied to data points
iteratively, and thus are suited to an
%With the approximations, we can regard the primal formulation with
%nonlinear kernels as another primal formulation with a linear kernel
%for transformed input vectors,
%% The approximation yields a {\em linear} formulation with transformed
%% feature vectors, which can be solved with the use of well known subgradient
%%approaches. 
%which can be solved using 
online context. Further, we suggest an efficient way to make
predictions for test points using the approximate feature mappings,
without recovering the potentially large number of support vectors.
%The approach has the added benefits that the approximate
%solution to the SVM yields an approximate classification function that
%can be evaluated cheaply, typically in time proportional to the
%dimension of approximation.

% We mention two more approaches that are related to the methods of this
% paper.  Fine \& Scheinberg~\cite{FinSch01} use a low-rank
% approximation to the Gram matrix to obtain a reduced dual formulation
% whose structure can be exploited by an interior-point solver. Rahimi
% \& Recht~\cite{RahR08a} exploit random projections to have simple
% least-squares formulation; we re-use their projection technique as
% part of our approach. Our method differs from these in that (i) we
% find an approximate solution of the primal SVM formulation in much
% less time than an interior-point method would require; and (ii) we use
% the SVM formulation rather than the least-squares formulation of
% \cite{RahR08a}, and (iii) out approach works both in batch and online
% settings.

\section{{Nonlinear SVMs in the Primal}}
\label{sec:nlsvm}
\noindent
In this section we discuss the primal SVM formulation in a
low-dimensional feature space induced by kernel approximation. 
% on the SVM primal formulation with
% nonlinear kernels, deriving a simple formulation 

% In this section
% % we develop a general theory for nonlinear SVM in the
% %primal form focusing on classification, then 
% we discuss how we can reformulate
% %it 
% the SVM primal formulation with nonlinear kernels 
% %reformulates 
% to a linear SVM problem, focusing on classification, by means of the
% low-dimensional approximation of a feature mapping.
% %  kernel. 
% We discuss techniques for approximating the kernel and finally how to
% classify data points efficiently.

\subsection{Structure of the Formulation} \label{sec:basic}

We first analyze the structure of the primal SVM formulation with
nonlinear feature mappings. To unveil the details, here we apply the
tools of convex analysis rigorously, rather than appealing to the
representer theorem~\citep{KimW70} as in \citet{Cha07}, where the idea
was first introduced.
%Although the idea was introduced by
%\citet{Cha07}, 
%
%
%. To unveil the details, we apply the tools of convex
%analysis rigorously, rather than replying on the representer
%theorem~\citep{KimW70} for justification as in Chapelle's work.
%%Here we use the tools of convex analysis to justify the primal SVM
%%formulation with kernels, which was first introduced by \citet{Cha07}
%%but no rigorous justification has been done to our knowledge.
%%

Let us consider the training point and label pairs $\{ (\vt_i, y_i)
\}_{i=1}^m$ for $\vt_i\in \IR^n$ and $y_i \in \IR$, and a feature
mapping $\tilphi: \IR^n \rightarrow \IR^d$. Given a convex loss function
$\ell(\cdot):\IR \rightarrow \IR \cup \{\infty\}$ and $\lambda>0$, the
primal SVM problem (for classification) can be stated as follows :
\begin{equation*}%\label{eq:a1}
\text{(P1)} \;\; \min_{\vw\in \IR^d, b\in \IR} \; \frac{\lambda}{2} \vw^T\vw 
                     + \frac{1}{m} \sum_{i=1}^m \ell(y_i (\vw^T \tilphi(\vt_i) + b))  .
\end{equation*}
%for some $\lambda>0$.% and a convex loss function $\ell$.
The necessary and sufficient optimality conditions are
\begin{subequations} \label{eq:a2}
\begin{align}
\lambda \vw + \frac{1}{m} \sum_{i=1}^m \chi_i y_i \tilphi(\vt_i) &= 0, \label{eq:a2.a} \\ 
\frac{1}{m} \sum_{i=1}^m \chi_i y_i &= 0, \label{eq:a2.b} \\ 
\text{for some} \; \chi_i \in \partial \ell\left(y_i(\vw^T \tilphi(\vt_i) + b)\right), &
\; i=1,2,\dotsc,m. \label{eq:a2.c}
\end{align}
\end{subequations}
where $\partial \ell$ is the subdifferential of $\ell$.

We now consider the following substitution:
\begin{equation} \label{eq:a3}
\vw = \sum_{i=1}^m \alpha_i \tilphi(\vt_i)
\end{equation}
(which mimics the form of \eqref{eq:a2.a}). Motivated by this
expression, we formulate the following problem
\begin{equation*}%\label{eq:a4}
\text{(P2)} \;\;  \min_{\alpha \in \IR^m, b\in \IR} \; 
\frac{\lambda}{2} \alpha^T \Psi \alpha   
+ \frac{1}{m} \sum_{i=1}^m \ell\left( y_i (\Psi_{i \cdot} \alpha + b)\right),
\end{equation*}
where $\Psi \in \IR^{m \times m}$ is defined by
\begin{equation}\label{eq:a5}
\Psi_{ij} := \tilphi(\vt_i)^T \tilphi(\vt_j), \;\; i,j = 1,2,\dotsc,m,
\end{equation}
and $\Psi_{i \cdot}$ denotes the $i$-th row of $\Psi$.  Optimality
conditions for (P2) %\eqref{eq:a4}
are as follows:
\begin{subequations} \label{eq:a6}
\begin{align}
\label{eq:a6.a}
\lambda \Psi \alpha + \frac{1}{m} \sum_{i=1}^m \beta_i y_i \Psi_{i \cdot}^T &= 0, \\
\label{eq:a6.b}
\frac{1}{m} \sum_{i=1}^m \beta_i y_i&= 0, \\
\label{eq:a6.c}
\makebox{for some} \; \beta_i \in \partial \ell\left( y_i(\Psi_{i \cdot}\alpha  + b)\right), &
\; i=1,2,\dotsc,m.
\end{align}
\end{subequations}
We can now derive the following result via convex analysis,
%The following result is a representer theorem for SVMs derived via
%convex analysis, 
showing that the solution of (P2) % \eqref{eq:a4}
can be used to derive a solution of (P1). %\eqref{eq:a1}.
This result can be regarded as a special case of the representer theorem.
%This observation is potentially interesting because
%(P2) %\eqref{eq:a4}
%is formulated in terms of the kernel $\Psi$ and does not require
%explicit knowledge of the feature mapping $\phi$.

\begin{proposition} \label{PROP:A1} Let $(\alpha,b) \in \IR^m \times
  \IR$ be a solution of (P2). %\eqref{eq:a4}.
  Then if we define $\vw$ by \eqref{eq:a3}, $(\vw ,b) \in \IR^d \times
  \IR$ is a solution of (P1). %\eqref{eq:a1}.
\end{proposition}

\begin{proof}
  Since $(\alpha,b)$ solves (P2), %\eqref{eq:a4},
  the conditions \eqref{eq:a6} hold, for some $\beta_i$,
  $i=1,2,\dotsc,m$. To prove the claim, it suffices to show that
  $(\vw,b)$ and $\chi$ satisfy \eqref{eq:a2}, where $\vw$ is defined
  by \eqref{eq:a3} and $\chi_i=\beta_i$ for all $i=1,2,\dotsc,m$.

  By substituting \eqref{eq:a5} into \eqref{eq:a6}, we have 
\begin{align*}
\lambda \sum_{i=1}^m \tilphi(\vt_j)^T \tilphi(\vt_i)  \alpha_i 
+ \frac{1}{m} \sum_{i=1}^m \beta_i y_i \tilphi(\vt_j)^T \tilphi(\vt_i) &= 0, \\
%\qquad\qquad j=1,2,\dotsc,m, & \\
\frac{1}{m} \sum_{i=1}^m \beta_i y_i &= 0, \\
\;\; \beta_i \in \partial \ell \left( y_i \left( \sum_{j=1}^m \tilphi(\vt_j)^T \tilphi(\vt_i) \alpha_j  + b \right)\right), \;\; i=1,2,&\dotsc, m.
\end{align*}
From the first equality above, we have that 
\[
-\sum_{i=1}^m  \left( \alpha_i + \frac{y_i}{\lambda m} \beta_i \right) \tilphi(\vt_i) + \xi=0,
\]
for some $\xi \in \makebox{Null} \left( \left[ \tilphi(\vt_j)^T
  \right]_{j=1}^m \right)$. Since the two components in this sum are
orthogonal, we have
\[
0 = \left\| \sum_{i=1}^m  \left( \alpha_i  + \frac{y_i}{\lambda m} \beta_i  \right)  
  \tilphi(\vt_i) \right\|_2^2 + \xi^T\xi,
\]
which implies that $\xi=0$. We can therefore rewrite the optimality
conditions for (P2) % \eqref{eq:a4}
as follows:
\begin{subequations}
\label{eq:a7}
\begin{align}
\label{eq:a7.a}
\sum_{i=1}^m \left( \lambda \alpha_i  + \frac{y_i}{m} \beta_i \right) \tilphi(\vt_i) &= 0, \\
\label{eq:a7.b}
\frac{1}{m} \sum_{i=1}^m \beta_i y_i &= 0,\\
\label{eq:a7.c}
 \beta_i \in \partial \ell\left(  y_i\left( \tilphi(\vt_i)^T \sum_{j=1}^m \alpha_j \tilphi(\vt_j)  + b \right)\right), & 
\;\; i=1,2,\dotsc,m. 
\end{align}
\end{subequations}
By defining $\vw$ as in \eqref{eq:a3} and setting $\chi_i = \beta_i$
for all $i$, we see that \eqref{eq:a7} is identical to \eqref{eq:a2},
as claimed.
\end{proof}

While $\Psi$ is clearly symmetric positive semidefinite, the proof
makes no assumption about nonsingularity of this matrix, or uniqueness
of the solution $\alpha$ of (P2). %\eqref{eq:a4}. 
% In fact, $\alpha$ may be nonunique. 
However, \eqref{eq:a6.a} suggests that without loss of generality, we
can constrain $\alpha$ to have the form
\[
 \alpha_i = - \frac{y_i}{\lambda m} \beta_i,
\]
where $\beta_i$ is restricted to $\partial \ell$. (For the hinge loss
function $\ell(\delta):=\max\{1-\delta,0\}$, we have $\beta_i \in [-1,
0]$.) These results clarify the connection between the expansion
coefficient $\alpha$ and the dual variable $\beta (=\chi)$, which is
introduced in \citet{Cha07} but not fully explicated there. Similar arguments for the
regression with the $\epsilon$-insensitive loss function $\ell'(
\delta ) := \max\{|\delta|-\epsilon, 0\}$ leads to
\[
% \alpha_i = - \frac{1}{\lambda m} (\beta_i' - \beta_i),
 \alpha_i' = - \frac{1}{\lambda m} \beta_i',
\]
%where $\beta_i \in \partial \ell( y_i - (\vw^T \tilphi(\vt_i)+b) )$
%and $\beta_i' \in \partial \ell( - y_i + (\vw^T \tilphi(\vt_i)+b) )$.
where $\beta_i' \in [-1,1]$ is in $\partial \ell'$.
%For the $\epsilon$-insensitive loss function
%$\ell(|\delta|)=\max\{|\delta|-\epsilon,0\}$ 
%we have $\beta_i \in [-1,1]$.
% In particular, if we use hinge loss function, that is,
% \begin{equation} \label{eq:l1loss}
% \ell(\delta) := \max \{0, 1-\delta\},
% \end{equation}
% then the subdifferential is
% \[
% \partial \ell(\delta) = \begin{cases}
% \{-1\} & \mbox{if $\delta<1$}, \\
% [-1,0] & \mbox{if $\delta=1$}, \\
% \{0\} & \mbox{if $\delta>1$}.
% \end{cases}
% \]
% Thus $\beta_i \in [-1,0]$ for all $i=1,2,\dotsc,m$.

%\subsection{Reformulation to a Linear SVM Problem} 
\subsection{Reformulation using Approximations} 
\label{sec:lowrank}

Consider the original feature mapping $\phiG:\IR^n \to \cH$ to a
Hilbert space $\cH$ induced by a kernel $\kG: \IR^n \times \IR^n \to
\IR$, where $\kG$ satisfies the conditions of Mercer's
Theorem~\citep{SchS01}.
% to guarantee the existence of $\phiG$ satisfying $\kG(\vs,\vt) =
% \langle \phiG(\vs), \phiG(\vt) \rangle$.
Suppose that we have a low-dimensional approximation $\phi:\IR^n \to
\IR^d$ of $\phiG$ for which
\begin{equation} \label{eq:kkG}
  \kG(\vs, \vt) \approx \phi(\vs)^T\phi(\vt),
\end{equation}
for all inputs $\vs$ and $\vt$ of interest. If we construct a matrix
$V \in \IR^{m\times d}$ for training examples $\vt_1,\vt_2,\dots,\vt_m$
by defining the $i$-th row as
\begin{equation}\label{eq:V}
 V_{i\cdot} = \phi(\vt_i)^T,\;\; i=1,2,\dots,m,
\end{equation}
we have that
\begin{equation} \label{eq:vvt}
\Psi := VV^T \approx \PsiG:=[\kG(\vt_i,\vt_j) ]_{i,j=1,2,\dotsc,m}.
\end{equation}
Note that $\Psi$ is a positive semidefinite rank-$d$ approximation to
$\PsiG$. By substituting $\Psi=VV^T$ in (P2), we obtain
\begin{equation}\label{eq:a4.alpha}
  \min_{\alpha \in \IR^m, b\in \IR} \; \frac{\lambda}{2} \alpha^T VV^T
  \alpha 
%  + \frac{1}{m} \sum_{i=1}^m \ell(\vv_i^T V^T \alpha + y_ib),
  + \frac{1}{m} \sum_{i=1}^m \ell(y_i(V_{i\cdot} V^T \alpha + b)).
\end{equation}
%where $\vv_i := V_{i \cdot}^T$. 
% is the transpose of the $i$-th row of $V$.
A change of variables
\begin{equation} \label{eq:ag}
\gamma = V^T \alpha
\end{equation}
leads to the equivalent formulation
\begin{equation*} %\label{eq:a4.gamma} 
\text{(PL)} \;\;  \min_{\gamma \in \IR^d, b\in  \IR} \; 
\frac{\lambda}{2} \gamma^T \gamma + \frac{1}{m}
%  \sum_{i=1}^m \ell(\vv_i^T \gamma + y_ib).
  \sum_{i=1}^m \ell(y_i(V_{i\cdot} \gamma + b)).
\end{equation*}
This problem can be regarded as a {\em linear} SVM with transformed
%feature vectors $y_i \vv_i \in \IR^d$, $i=1,2,\dotsc,m$.
feature vectors $ V_{i\cdot}^T \in \IR^d$, $i=1,2,\dotsc,m$.
% obtained by the projection onto the $d$ dimensional subspace.
An approximate solution to (PL) can be obtained with the subgradient
algorithms discussed later in Section~\ref{sec:ss}.
% We can solve (PL) by applying
%linear SVM techniques to find $(\gamma,b)$. 

Any $\alpha \in \IR^m$ that solves the overdetermined system
\eqref{eq:ag} will yield a solution of \eqref{eq:a4.alpha}.  (Note
that $\alpha$ satisfying \eqref{eq:ag} need have at most $d$
nonzeros.) In Section~\ref{sec:testing}, we will discuss an efficient
way to make predictions without recovering $\alpha$.

\subsection{Approximating the Kernel} \label{sec:approxK}

We discuss two techniques for finding $V$ that satisfies
\eqref{eq:vvt}. The first uses randomized linear algebra to calculate
a low-rank approximation to the Gram matrix $\PsiG$.
% as in \eqref{eq:vvt}.
The second approach uses random projections to construct approximate
feature mappings $\tilphi$ explicitly.
% approximates the feature mapping $\phiG$ explicitly by approximate
% feature mapping $\tilphi$ constructed using random projections.

\subsubsection{Kernel Matrix Approximation}\label{sec:kernel.matrix.approx}

Our first approach makes use of the Nystr\"om sampling
idea~\citep{DriM05}, to find a good approximation of specified rank
$d$ to the $m \times m$ matrix $\PsiG$ in \eqref{eq:vvt}. In this
approach, we specify some integer $s$ with $0 < d \le s < m$, and
choose $s$ elements at random from the index set $\{1,2,\dotsc,m \}$
to form a subset $\cS$.  We then find the best rank-$d$ approximation
$W_{\cS,d}$ to $(\PsiG)_{\cS \cS}$, and its pseudo-inverse
$W_{\cS,d}^+$. We choose $V$ so that
\begin{equation} \label{eq:ny.1}
VV^T = (\PsiG)_{\cdot \cS} W_{\cS,d}^+ (\PsiG)_{\cdot \cS}^T,
\end{equation}
where $(\PsiG)_{\cdot \cS}$ denotes the column submatrix of $\PsiG$
defined by the indices in $\cS$. The results in \cite{DriM05} indicate
that in expectation and with high probability, the rank-$d$
approximation obtained by this process has an error that can be made
as close as we wish to the {\em best} rank-$d$ approximation by
choosing $s$ sufficiently large.

To obtain $W_{\cS,d}$, we form the eigen-decomposition $(\PsiG)_{\cS
  \cS} = QDQ^T$, where $Q \in \IR^{s \times s}$ is orthogonal and $D$
is a diagonal matrix with nonincreasing nonnegative diagonal
entries. Taking $\bar{d} \le d$ to be the number of positive diagonals
in $D$, we have that
\[
W_{\cS,d} =  Q_{\cdot, 1..\bar{d}} D_{1..\bar{d},1..\bar{d}} Q_{\cdot, 1..\bar{d}}^T,
\]
(where $Q_{\cdot, 1..\bar{d}}$ denotes the first $\bar{d}$ columns of
$Q$, and so on). The pseudo-inverse is thus
\begin{align*}%\label{eq:W.inv}
W_{\cS,d}^+ =  Q_{\cdot, 1..\bar{d}} D_{1..\bar{d},1..\bar{d}}^{-1} Q_{\cdot, 1..\bar{d}}^T,
\end{align*}
and the matrix  $V$ satisfying \eqref{eq:ny.1} is therefore
\begin{equation} \label{eq:ny.v.qdq}
V = (\PsiG)_{\cdot \cS} Q_{\cdot, 1..\bar{d}} D_{1..\bar{d},1..\bar{d}}^{-1/2}.
\end{equation}

For practical implementation, rather than defining $d$ a priori, we
can choose a threshold $\epsilon_d$ with $0 < \epsilon_d \ll 1$, then
choose $d$ to be the largest integer in $1,2,\dotsc,s$ such that
$D_{dd} \ge \epsilon_d$. (In this case, we have $\bar{d}=d$.)

For each sample set $\cS$, this approach requires $O(ns^2+s^3)$
operations for the creation and factorization of $(\PsiG)_{\cS\cS}$,
% and $O(sm)$ for the computation of $(\PsiG)_{\cdot,\cS}$,
assuming that the evaluation of each kernel entry takes $O(n)$
time. Since our algorithm only requires a single row of $V$ in each
iteration, the computation cost of \eqref{eq:ny.v.qdq} can be
amortized over iterations: the cost is $O(sd)$ per iteration if the
corresponding row of $\PsiG$ is available; $O(ns + sd)$ otherwise.
% The time complexity of the
% kernel approximation discussed above is $O(s^3+sm(n+d))$, comprising
% (i) a cost of $O(s^3)$ for the $QDQ^T$ factorization of
% $(\PsiG)_{\cS\cS}$, (ii) $O(smn)$ for computation of
% $(\PsiG)_{\cdot,\cS}$, and (iii) $O(smd)$ for the matrix
% multiplication of \eqref{eq:ny.v.qdq}. Note that the cost of (ii) and
% (iii) dominates the cost of (i) since $d\leq s \ll m$.

\subsubsection{Feature Mapping Approximation}\label{sec:kernel.func.approx}

The second approach to defining $V$ finds a mapping $\tilphi: \IR^n
\rightarrow \IR^d$ that satisfies
\[
 \langle \phiG(\vs), \phiG(\vt) \rangle = \Exp \left[ \langle \tilphi(\vs), \tilphi(\vt) \rangle \right] ,
\]
where the expectation is over the random variables that determine
$\tilphi$. The approximate mapping $\tilphi$ can be constructed
explicitly by random projections as follow~\citep{RahR08a},
% . Following \cite{RahR08a}, we write
% \begin{align}\label{eq:explicit.phi}
%  \tilphi(\vt) = \sqrt{\frac{1}{d}} \left[ \cos(\nu_1^T\vt+\gamma_1),\cdots,\cos(\nu_d^T\vt+\gamma_d) \right]^T
% \end{align}
\begin{align}\label{eq:explicit.phi}
 \tilphi(\vt) = \sqrt{\frac{2}{d}} \left[ \cos(\nu_1^T\vt+\omega_1),\cdots,\cos(\nu_d^T\vt+\omega_d) \right]^T
\end{align}
where $\nu_1,\dots,\nu_d \in \IR^n$ are i.i.d. samples from a
distribution with density $p(\nu)$, and $\omega_1,\dots,\omega_d\in\IR$
are from the uniform distribution on $[0,2\pi]$. The density function
$p(\nu)$ is determined by the types of the kernels we want to use.
For the Gaussian kernel
\begin{align}\label{eq:gaussk}
 \kG(\vs,\vt) = \exp (-\sigma \|\vs-\vt\|_2^2 ),
\end{align}
we  have
% \[ 
%  p(\nu) = \frac{1}{(2\pi)^{d/2}\sigma^2} \exp \left(
% -\frac{||\nu||_2^2}{2\sigma^2} \right),
% \]
\[ 
 p(\nu) = \frac{1}{(4\pi\sigma)^{d/2}} \exp \left(
-\frac{||\nu||_2^2}{4\sigma} \right),
\]
from the Fourier transformation of $\kG$. 
% This is the density function of a multinomial Gaussian distribution
% with zero mean and the covariance matrix of
% $\Sigma=\text{diag}(2\sigma,2\sigma,\dots,2\sigma)$.  
%% We can define the matrix $V$ satisfying \eqref{eq:vvt} by setting
%% \begin{equation} \label{eq:benV} V_{i\cdot}^T = y_i \phi(\vt_i),
%%   \;\; i=1,2,\dotsc,m,
%% \end{equation}
%% thus setting up the formulation (PL).

This approximation method is less expensive than the previous
approach, requiring only $O(nd)$ operations for each data point
(assuming that sampling of each vector $\nu_i \in \IR^n$ takes $O(n)$
time).
%This method is suitable for online settings, since it takes
%only $O(nd)$ to prepare $\nu_1,\dots,\nu_d$ (assuming that sampling
%each component of these vectors takes constant time) and $O(d)$ to
%process each data point. 
As we observe in Section~\ref{sec:computation}, however, this approach
tends to give lower prediction accuracy than the first approach
for a fixed $d$ value.

\subsection{Efficient Prediction}\label{sec:testing}

Given the solution $(\gamma,b)$ of (PL), we now describe how the
prediction of a new data point $\vt \in \IR^n$ can be made efficiently
without recovering the support vector coefficient $\alpha$ in
(P2). The imposed low dimensionality of the approximate kernel in our
approach can lead to significantly lower cost of prediction, as low as
a fraction of $d/\text{(no. support vectors)}$ of the cost of an
exact-kernel approach.

For the feature mapping approximation of
Section~\ref{sec:kernel.func.approx}, 
% where $\phi$ is defined explicitly by \eqref{eq:explicit.phi},
we can simply use the decision function $f$ suggested immediately by
(P1), that is, $f(\vt) = \vw^T\phi(\vt) + b$.  Using the definitions
\eqref{eq:a3}, \eqref{eq:V}, and \eqref{eq:ag},
% \eqref{eq:benV}, 
we obtain
\begin{align*}
f(\vt) &=
\phi(\vt)^T \sum_{i=1}^m \alpha_i \phi(\vt_i) + b \\
&= \phi(\vt)^T V^T \alpha + b = \phi(\vt)^T\gamma + b.
\end{align*}
%where we used \eqref{eq:ag} for the final equality. 
The time complexity in this case is $O(nd)$.
% Note in particular that the classifier can be evaluated directly
% from $\gamma$; there is no need to recover $\alpha$ explicitly.

For the kernel matrix approximation approach of
Section~\ref{sec:kernel.matrix.approx}, the decision function
$\vw^T\phi(\vt) + b$ cannot be used directly, as we have no way to
evaluate $\phi(\vt)$ for an arbitrary point $\vt$. We can however use
the approximation \eqref{eq:kkG} to note that
\begin{align}
\nonumber
\phi(\vt)^T \vw + b &= \sum_{i=1}^m \alpha_i \phi(\vt)^T \phi(\vt_i)+b  \\
\label{eq:ka.class}
& \approx \sum_{i=1}^m \alpha_i \kG(\vt_i,\vt) + b,
\end{align}
so we can define the function \eqref{eq:ka.class} to be the decision
function. To evaluate this, we need only compute those kernel values
$\kG(\vt_i,\vt)$ for which $\alpha_i \neq 0$. As noted in
Section~\ref{sec:lowrank}, we can satisfy \eqref{eq:ag} by using just
$d$ nonzero components of $\alpha$, so \eqref{eq:ka.class} requires
only $d$ kernel evaluations.

If we set $\alpha_i=0$ for all components $i \notin \cS$, where $\cS$
is the sample set from Section~\ref{sec:approxK} and $s=d=\bar d$, we
can compute $\alpha$ that approximately satisfies \eqref{eq:ag}
without performing further matrix factorizations. Denoting the nonzero
subvector of $\alpha$ by $\alpha_\cS$, we have $V^T \alpha = V_{\cS
  \cdot}^T \alpha_\cS = \gamma$, so from \eqref{eq:ny.v.qdq} and the
fact that $(\PsiG)_{\cS \cS} = QDQ^T$, we have
\begin{align*}
\gamma = \left[ (\PsiG)_{\cS \cS} Q_{\cdot, 1..\bar{d}} D_{1..\bar{d},1..\bar{d}}^{-1/2} \right]^T \alpha_\cS 
% &= D_{1..\bar{d},1..\bar{d}}^{-1/2} Q_{\cdot, 1..\bar{d}}^T (\PsiG)_{\cS \cS} \alpha_\cS 
= D_{1..\bar{d},1..\bar{d}}^{1/2}
Q_{\cdot, 1..\bar{d}}^T \alpha_\cS.
\end{align*}
That is, $\alpha_\cS = Q_{\cdot, 1..\bar{d}}
D_{1..\bar{d},1..\bar{d}}^{-1/2} \gamma$, which can be computed in
$O(d^2)$ time. Therefore, prediction of a test point will take $O(d^2
+ nd)$ for this approach, including kernel evaluation time.

\section{{Stochastic Approximation Algorithm}}
\label{sec:ss}
\noindent
We describe here a stochastic approximation algorithm for solving the
linear SVM reformulation (PL). Consider the general
convex optimization problem
\begin{equation} \label{eq:fX} 
\min_{x \in X} \, f(x),
\end{equation}
where $f$ is a convex function and $X$ is a bounded closed convex
set with the radius $D_X$ defined by
\begin{align}
  D_X := \max_{x \in X} ||x||_2 . \label{eq:DX}
\end{align}
We use $g(x)$ to denote a particular subgradient of $f(x)$. By
convexity of $f$, we have
\[
  f(x') - f(x) \geq g(x)^T (x'-x), \;\; \forall x, x' \in X , \;
\forall g(x) \in \partial f(x).
\]
$f$ is {\em strongly} convex when there exists $\mu>0$ such that
\begin{equation*}% \label{eq:sc}
(x' - x)^T \left[ g(x') - g(x) \right] \geq \mu ||x' - x||^2, 
\end{equation*}
for all $x, x' \in X$, all $g(x) \in \partial f(x)$, and all  $g(x') \in
\partial f(x')$.
Note that the objective in (PL) is strongly convex in $\gamma$, but
only convex in $b$. Pegasos~\citep{ShaSSS07} requires $f$ to be
strongly convex in all variables and thus modifies the SVM formulation
to have this property. The approach we describe below is suitable for
the original SVM formulation.

\subsection{The ASSET Algorithm}

% Suppose that the arrival of data points is governed by i.i.d. random
% variable $\xi$. 
Our algorithm assumes that at any $x \in X$, we have available\
$G(x;\xi)$, a stochastic subgradient estimate depending on random
variable $\xi$ that satisfies $\E [G(x;\xi)] = g(x)$ for some $g(x)
\in \partial f(x)$. The norm deviation of the stochastic subgradients
is measured by $D_G$ defined as follows:
\begin{equation} \label{eq:DG} 
\E[ \| G(x;\xi) \|_2^2 ] \le D_G^2 \quad \forall x\in X, \xi\in\Xi.
\end{equation}

\paragraph{Iterate Update:}
At iteration $j$, the algorithm takes the following step:
\begin{equation*}%\label{eq:update-ssa}
  x^{j} = \Pi_{X}( x^{j-1} - \eta_j G(x^{j-1}; \xi^j) ), \;\; j=1,2,\dotsc,
\end{equation*}
where $\xi^j$ is a random variable (i.i.d. with the random variables
used at previous iterations), $\Pi_X$ is the Euclidean projection onto
$X$, and $\eta_j>0$ is a step length. For our problem (PL), we have
$x^j=(\gamma^j,b^j)$, and $\xi^j$ is selected to be one of the indices
$\{1,2,\dotsc,m\}$ with equal probability, and the subgradient
estimate is constructed from the subgradient for the $\xi^j$th term in
the summation of the empirical loss term. Table~\ref{tbl:subgrad}
summarizes the subgradients $G(x^{j-1}; \xi^j)$ for classification and
regression tasks, with the hinge loss and the $\epsilon$-insensitive
loss functions respectively.
\begin{table*}[th!]
\centering
\caption{Loss functions and their corresponding subgradients for classification and regression tasks.\label{tbl:subgrad}}
\begin{tabular}{|c@{\sg}||c@{\sg}|l@{\sg}l@{\sg}|}\hline
 Task          &  Loss Function, $\ell$                      
               & \multicolumn{2}{c|}{Subgradient,  
$G \left( \begin{bmatrix} \gamma^{j-1} \\ b^{j-1} \end{bmatrix};\xi^j \right)$ }         \\\hline
Classification & $\max \{1-y(\vw^T\tilphi(\vt)+b), 0 \}$
               & $ \begin{bmatrix} \lambda \gamma^{j-1} + d_j V_{{\xi^j} \cdot}^T\\  d_j  \end{bmatrix}$,
               & $ d_j = \begin{cases} -y_{\xi^j}   & \text{if $ y_{\xi^j}(V_{\xi^j\cdot}\gamma^{j-1} + b^{j-1}) < 1$}\\
                                        0   & \text{otherwise} \end{cases}$ \\ \hline
Regression    &  $\max \{|y - (\vw^T\tilphi(\vt)+b)| - \epsilon, 0\}$
              & $ \begin{bmatrix} \lambda \gamma^{j-1} + d_j V_{\xi^j\cdot}^T\\ d_j  \end{bmatrix}$,
               & $ d_j=\begin{cases} -1 & \text{if $y_{\xi^j} > V_{\xi^j\cdot}\gamma^{j-1} + b^{j-1} + \epsilon$,} \\
                                      1 & \text{if $y_{\xi^j} < V_{\xi^j\cdot}\gamma^{j-1} + b^{j-1} -\epsilon$,}\\
                                      0 & \text{otherwise.}\end{cases}$\\\hline
\end{tabular}
\end{table*}
% Specifically, when the hinge loss function is used for
% classification, i.e.,
% \begin{equation*}% \label{eq:l1loss}
% \ell(\vw^T\phi(\vt)+b,y) := \max \{0, 1-y(\vw^T\tilphi(\vt)+b) \},
% \end{equation*}
% %from \eqref{eq:l1loss} 
% then we have
% \begin{equation}\label{eq:svm-subgrad}
%  G \left(\begin{bmatrix} \gamma^j \\ b^j \end{bmatrix};\xi^j \right) = 
% \begin{bmatrix} \lambda \gamma^j + d_j V_{{\xi^j} \cdot}^T\\
%                 d_j y_{\xi^j}
% \end{bmatrix},
% \end{equation}
% where $d_j=-1$ if the kernelized training point $y_{\xi^j}V_{{\xi^j} \cdot}^T$ is
% currently misclassified and $d_j=0$ otherwise. If we use the
% $\epsilon$-insensitive loss function for regression tasks, i.e.,
% \[
%  \ell(\vw^T\phi(\vt)+b,y) := \max \{0, | (\vw^T\tilphi(\vt)+b)| - \epsilon\},
% \]
% then the subgradient is defined as follows:
% \begin{equation*}%\label{eq:svm-subgrad}
%  G \left(\begin{bmatrix} \gamma^j \\ b^j \end{bmatrix};\xi^j \right) = 
% \begin{bmatrix} \lambda \gamma^j + d_j V_{\xi^j\cdot}^T\\
%                 d_j y_{\xi^j}
% \end{bmatrix},
% \end{equation*}
% where $d_j=-y_{\xi^j}$ if $y_{\xi^j} > y_{\xi^j}V_{\xi^j\cdot}\gamma + b +
% \epsilon$, $d_j=y_{\xi^j}$ if $y_{\xi_j} <
% y_{\xi^j}V_{\xi^j\cdot}\gamma+b-\epsilon$, and $d_j=0$ otherwise. 

\paragraph{Feasible Sets:}
We define the feasible set $X$ to be the Cartesian product of a ball
in the $\gamma$ component with an interval $[-B,B]$ for the $b$
component. The following shows the set $X$ for classification, for
which the radius of the ball is derived using strong
duality~\citep[Theorem 1]{SSS11}:
\[
 X = \left\{ \begin{bmatrix}\gamma\\ b\end{bmatrix} \in \IR^d \times \IR :\; ||\gamma||_2 \leq 1/\sqrt{\lambda},\; |b| \leq B \right\} 
\]
for sufficiently large $B$, resulting in $D_X = \sqrt{1/\lambda +B^2}$.
For regression, the following theorem provides a radius for $\gamma$:
\begin{theorem}
  For SVM regression using the $\epsilon$-insensitive loss function with
  $0\le\epsilon<\|\mathbf{y}\|_\infty$, where $\mathbf{y} :=
  (y_1,y_2,\dots,y_m)^T$, we have
\[
  \|\gamma\|_2 \le \sqrt{\frac{2(\|\mathbf{y}\|_\infty -\epsilon)}{\lambda }}.
\]
\end{theorem}
\begin{proof}
 We can write an equivalent formulation of (PL) as follows:
\[
 \min_{\gamma,b} \;\; \frac{1}{2} \gamma^T\gamma + C \sum_{i=1}^m \max\{ |y_i-(\gamma^T\tilphi(\vt_i)+b)|-\epsilon, 0\},
\]
for $C=1/(\lambda m)$. The corresponding Lagrange dual formulation is
\begin{align*}
 \max_{z,z'} \; & -\frac{1}{2} \sum_{i=1}^m\sum_{j=1}^m (z_i' - z_i)(z_j'-z_j)\langle \tilphi(\vt_i),\tilphi(\vt_j)\rangle \\ 
& - \epsilon \sum_{i=1}^m (z_i'+z_i) + \sum_{i=1}^m y_i (z_i'-z_i) \\
 \text{s.t.} \; & \sum_{i=1}^m (z_i' - z_i) = 0, \\
             & 0\le z_i \le C, \; 0 \le z_i' \le C, \; i=1,2,\dots,m.
\end{align*}
Let $(\gamma^*,b^*)$ and $(z^*,z'^*)$ be the optimal solutions of the
primal and the dual formulations, respectively. Also, from the KKT
conditions we have $\gamma^* = \sum_{i=1}^m (z_i'^* - z_i^*)
\tilphi(\vt_i)$. Replacing this in the optimal dual objective, and
using strong duality, we have
\begin{align*}
& \frac{1}{2} (\gamma^*)^T\gamma^*\\ 
&\le \frac{1}{2} (\gamma^*)^T\gamma^* + C \sum_{i=1}^m \max\{ |y_i-((\gamma^*)^T\tilphi(\vt_i)+b^*)|-\epsilon, 0\}\\
&= -\frac{1}{2} (\gamma^*)^T\gamma^*  - \epsilon \sum_{i=1}^m (z_i'^*+z_i^*) + \sum_{i=1}^m y_i (z_i'^*-z_i^*) \\
%&\le -\frac{1}{2} (\gamma^*)^T\gamma^*  - \epsilon \sum_{i=1}^m (z_i'^*+z_i^*) + \| \mathbf{y} \|_\infty \sum_{i=1}^m  (z_i'^*+z_i^*)\\
&\le -\frac{1}{2} (\gamma^*)^T\gamma^*  + 2 (\| \mathbf{y} \|_\infty - \epsilon ) \|z\|_1.
\end{align*}
Since $0 \le z_i \le C$, we have $\|z\|_\infty \le C$ and thus
$\|z\|_1 \le Cm = 1/\lambda$. Applying this to the above inequality
leads to our claim. We exclude the case $\epsilon \ge
\|\mathbf{y}\|_\infty$, where the optimal solution is trivially
$(\gamma^*,b^*)=(\vzero,0)$.
\end{proof}

\paragraph{Averaged Iterates:}
The solution of \eqref{eq:fX} is estimated not by the iterates $x^j$
but rather by a weighted sum of the final few iterates. Specifically,
if we define $N$ to be the total number of iterates to be used and
$\bar{N}<N$ to be the point at which we start averaging, the final
reported solution estimate would be
\begin{equation*}%\label{eq:avg-iterate}
  \tilde x^{\bar{N},N} := \frac{\sum_{t=\bar{N}}^N  \eta_t x^t}{\sum_{t=\bar{N}}^N  \eta_t} .
\end{equation*}
These is no need to store all the iterates $x^t$, $t=\bar{N},
\bar{N}+1, \dotsc,N$ in order to evaluate the average. Instead, a
running average can be maintained over the last $N-\bar{N}$
iterations, requiring the storage of only a single extra vector.

\paragraph{Estimation of $D_G$:}
The steplength $\eta_j$ requires knowledge of the subgradient estimate
deviation $D_G$ defined in \eqref{eq:DG}. We use a small random sample
of training data indexed by $\xi^{(l)}$, $l=1,2,\dotsc,M$, at the
first iterate $(\gamma^0,b^0)$, and estimate $D_G^2$ as
\begin{align*}
  E\left[ \left|\left|G\left( \begin{bmatrix} \gamma^0\\ b^0\end{bmatrix};\xi \right) \right|\right|_2^2 \right] 
\approx \frac{1}{M} \sum_{l=1}^M d_l^2 (||V_{\xi^{(l)}\cdot}||_2^2 + 1).
\end{align*}

We summarize this framework in Algorithm~\ref{alg:ssa} and refer it as
ASSET.
% The algorithm can be easily modified for regression by chaging the
% steps 8 and 9 according to Table~\ref{tbl:subgrad}.
The integer $\bar{N}>0$ specifies the iterate at which the algorithm
starts averaging the iterates, which can be set to $1$ to average all
iterates, to a predetermined maximum iteration number to output the
last iterate without averaging, or to a number in between.

\begin{algorithm}[t]
\caption{ASSET Algorithm}\label{alg:ssa}
\begin{algorithmic}[1]
  \STATE{Input: $T=\{(\vt_1,y_1),\dots,(\vt_m,y_m)\}$, $\PsiG$,
    $\lambda$, positive integers $\bar{N}$ and $N$ with $0 < \bar{N} <
    N$, and $D_X$ and $D_G$ satisfying \eqref{eq:DX} and
    \eqref{eq:DG};}
 \STATE{Set $(\gamma^0,b^0) = (\vzero, 0)$, $(\tilde\gamma, \tilde{b}) = (\vzero, 0)$, $\tilde\eta= 0$; }
 \FOR{$j=1,2,\dotsc,N$}
   \STATE{$\eta_j = \frac{D_X}{D_G\sqrt{j}}.$}
   \STATE{Choose $\xi^j \in \{1,\dots,m\}$ at random.}
   \STATE{$ 
          V_{\xi^j\cdot} = \begin{cases} 
                   V_{\xi^j \cdot} & \text{for $V$ as in \eqref{eq:ny.v.qdq}, or }\\
                   \tilphi(\vt_{\xi^j}) & \text{for $\tilphi(\cdot)$ as in \eqref{eq:explicit.phi} .}
                        \end{cases}
         $}
   \STATE{Compute  $G \left(\begin{bmatrix} \gamma^{j-1} \\ b^{j-1} \end{bmatrix};\xi^j \right)$ following Table~\ref{tbl:subgrad}.}
%   \STATE{$d_j = \left\{ \begin{array}{ll} -1 & \mbox{if} \;\; \vv_{\xi^j} \gamma^j + y_{\xi^j}b < 1 \\ 0 & \mbox{otherwise} \end{array} \right.$ }
%   \STATE{$\begin{bmatrix} \gamma^{j}\\ b^{j} \end{bmatrix} = 
%           \Pi_{X}\left( \begin{bmatrix}  (1-\eta_j\lambda) \gamma^{j-1} - \eta_j d_j \vv_{\xi^j}\\
%                            b^{j-1} - \eta_j d_j y_{\xi^j} \end{bmatrix}\right)$}
   \STATE{$\begin{bmatrix} \gamma^{j}\\ b^{j} \end{bmatrix} = 
           \Pi_{X}\left( \begin{bmatrix} \gamma^{j-1} \\ b^{j-1} \end{bmatrix} 
            - \eta_j G \left(\begin{bmatrix} \gamma^{j-1} \\ b^{j-1} \end{bmatrix};\xi^j \right) \right)$.}
   \IF{$j\geq \bar{N}$} 
\STATE{ \COMMENT{update averaged iterate}
\begin{align*}
\begin{bmatrix} \tilde\gamma\\ \tilde{b} \end{bmatrix} 
&= \frac{\tilde\eta}{\tilde\eta + \eta_j} \begin{bmatrix} \tilde\gamma\\ \tilde b \end{bmatrix} 
+ \frac{\eta_{j}}{\tilde\eta + \eta_j} \begin{bmatrix} \gamma^{j}\\ b^{j} \end{bmatrix} . \\
\tilde \eta &= \tilde \eta + \eta_j.
\end{align*}
}
\ENDIF
%\label{alg:ssa:average}
\ENDFOR
\STATE{Define $\tilde{\gamma}^{\bar{N},N} := \tilde\gamma$ and
  $\tilde{b}^{\bar{N},N} := \tilde b$.}
\end{algorithmic}
\end{algorithm}

\subsection{Convergence}
The analysis of robust stochastic approximation~\citep{NemJ09, NemY83}
provides theoretical support for the algorithm above.  Considering
Algorithm~\ref{alg:ssa} applied to the general formulation
\eqref{eq:fX}, and denoting the algorithm's output
$\tilde{x}^{\bar{N},N}$, we have the following result.

\begin{theorem}\label{thm:conv-ssa}
  Given the output $\tilde{x}^{\bar{N},N}$ and optimal function value
  $f(x^*)$, Algorithm~\ref{alg:ssa} satisfies
 \begin{equation*}% \label{eq:ssa-w}
 E[ f(\tilde{x}^{\bar{N},N}) - f(x^*) ] \leq C(\rho)
\frac{D_X D_G}{\sqrt{N}}
 \end{equation*}
  where $C(\rho)$ solely depends on the fraction
 $\rho \in (0,1)$ for which $\bar{N}=\lceil \rho  N\rceil$.
\end{theorem}

% Because of the length limitation, we refer to \cite{NemJ09} for the
% proof of more general form of this result.

\subsection{Strongly Convex Case}

Suppose that we omit the intercept $b$ from the linear formulation
(PL). Then its objective function $f(x)$ becomes strongly convex for
all of its variables. In this special case we can apply different
steplength $\eta_j = 1/(\lambda j)$ to achieve faster convergence in
theory.  The algorithm remains the same as Algorithm~\ref{alg:ssa}
except that averaging is no longer needed and a faster convergence
rate can be proved -- essentially a rate of $1/j$ rather than
$1/\sqrt{j}$ (see \cite{NemJ09} for a general proof):
\begin{theorem}\label{thm:conv-st}
  Given the output $x^N$ and optimal function value $f(x^*)$,
  Algorithm~\ref{alg:ssa} with $\eta_j = 1/(\lambda j)$ satisfies
 \begin{equation*}% \label{eq:ssa-w}
 E[ f(x^N) - f(x^*) ] \leq \max\left\{ \left( \frac{D_G}{\lambda} \right)^2, \;D_X^2 \right\} / N .
 \end{equation*}
\end{theorem}
\noindent
Note that when $\lambda \approx 0$, that is, when the strong convexity
is very weak, the convergence of this approach can be very slow
unless we have $D_G \approx 0 $ as well.

Without the intercept $b$, the feasible set $X$ is simplified only for
the $\gamma$ component, 
%as follows (for classification)
%\[
% X = \{ \gamma \in \IR^d : ||\gamma||_2 \leq 1/\sqrt{\lambda} \} ,
% \]
and the update steps are changed accordingly.
% to omit the component $b$.
The resulting algorithm, we refer is as {ASSET$^*$\xspace}, is the
same as Pegasos~\citep{ShaSSS07} and SGD~\citep{Bot05a}, except for
our extensions to nonlinear kernels.

% Note that averaging like \eqref{eq:avg-iterate} may still be useful,
% as it can be shown to improve the convergence rate by some constant
% \citep{PolJ92}.

\section{{Computational Results}} \label{sec:computation}

% \begin{figure*}[ht!] %[ht!]
% \centering
% \caption{The effect of the approximation dimension $d$ to the test
%   error. The x-axis shows the values of $s$ in log scale (base 2). For
%   {\assetonline}, $d=s$ and for the others $d\leq
%   s$.\label{fig:rank-testerr}}
%   \subfigure[\texttt{ADULT}]{\includegraphics[scale=.23, trim=8px 15px 8px 10px]{figures/adult.eps}\label{fig:rank-adult}}
%   \subfigure[\texttt{MNIST}]{\includegraphics[scale=.23, trim=8px 15px 8px 10px]{figures/mnist}\label{fig:rank-mnist}}
%   \subfigure[\texttt{CCAT}]{\includegraphics[scale=.23, trim=15px 15px 15px 15px]{figures/ccat}\label{fig:rank-ccat}} 
%   \subfigure[\texttt{IJCNN}]{\includegraphics[scale=.23, trim=15px 15px 15px 15px]{figures/ijcnn}\label{fig:rank-ijcnn}}
%   \subfigure[\texttt{COVTYPE}]{\includegraphics[scale=.23, trim=0px 15px 15px 15px]{figures/covtype}\label{fig:rank-covtype}}
% \end{figure*}

\begin{table*}[ht!]
  \caption{Data sets and training parameters.\label{tbl:params}}
  \center
%\begin{tabular}{|l@{\sg}|r@{\sg}|r@{\sg}|r@{\sg}r@{\sg}|c@{\sg}|}\toprule
%\begin{tabular}{l@{\sg}r@{\sg}r@{\sg}r@{\sg}r@{\sg}c@{\sg}}\toprule
\begin{tabular}{|l||rrrrcc|}\hline
 Name                   & m (train) &  valid/test &n & (density) & $\lambda$ & $\sigma$\\ \hline
 \texttt{ADULT}         &   32561   & 8140/8141 &123 &(11.2\%) &  3.07e-08 &  0.001\\ %1/(m\lambda)=1000
 \texttt{MNIST}         &   58100   & 5950/5950 &784 &(19.1\%)   &  1.72e-07 & 0.01\\ %100
 \texttt{CCAT}          &   78127   & 11575/11574 &47237 &(1.6\%)&  1.28e-06 & 1.0\\ %10
 \texttt{IJCNN}         &  113352   & 14170/14169 &22 &(56.5\%)  &  8.82e-08 & 1.0\\ %100
 \texttt{COVTYPE}       &  464809   & 58102/58101 &54 &(21.7\%)  &  7.17e-07 & 1.0\\ %3.0
 \texttt{MNIST-E}       & 1000000  & 20000/20000& 784 & (25.6\%) &  1.00e-08 & 0.01\\ %1.0
% \texttt{HUMAN}         & 1,000,000  &&  564 &(24.9\%) &  100, 0.01\\
\hline
\end{tabular}
\end{table*}

\noindent
We implemented our algorithms based on the open-source Pegasos
code\footnote{Our code is available at
  \url{http://pages.cs.wisc.edu/~sklee/asset/} and Pegasos is from
  \url{http://mloss.org/software/view/35/}.}.  We refer our algorithms
with kernel matrix approximation as {\asset} and {\assetst} (for the
versions that requires convexity and strong convexity, respectively)
and with feature mapping approximation as {\assetonline} and
{\assetstonline}.  In the interests of making direct comparisons with
other codes, we do not include intercept terms in our experiments,
since some of the other codes do not allow such terms to be used
without penalization.

We run all experiments on load-free 64-bit Linux systems with 2.66~GHz
processors and 8~GB memory. Kernel cache size is set to 1~GB when
applicable.  All experiments with randomness are repeated 50 times
unless otherwise specified.

Table~\ref{tbl:params} summarizes the six binary classification tasks
we use for the experiments\footnote{\texttt{ADULT}, \texttt{MNIST},
  \texttt{CCAT} and \texttt{COVTYPE} data sets are from the UCI
  Repository, \url{http://archive.ics.uci.edu/ml/}.}.
%Our experiments use six binary classification tasks.  
The \texttt{ADULT} data set is randomly split into
training/validation/test sets. In the \texttt{MNIST} data set, we
obtain a binary problem by classifying the digits 0-4 versus 5-9. In
\texttt{CCAT} from the RCV1 collection~\citep{rcv1}, we use the
original test set as the training set, and divide the original
training set into validation and test sets.  \texttt{IJCNN} is
constructed by a random splitting of the IJCNN 2001 Challenge
data set\footnote{\url{http://www.csie.ntu.edu.tw/~cjlin/libsvmtools/datasets/}}.  In
\texttt{COVTYPE}, the binary problem is to classify type 1 against the
other forest cover types.
%Finally
%\texttt{HUMAN} is for detecting splicing sites on human
%DNA\footnote{\url{http://www.fml.tuebingen.mpg.de/raetsch/suppl/lsmkl/}}. We
%randomly choose 1 million/20000/20000 split from the original 15
%million DNA sites.
Finally, \texttt{MNIST-E} is an extended set of \texttt{MNIST},
generated with elastic deformation of the original
digits\footnote{\url{http://leon.bottou.org/papers/loosli-canu-bottou-2006/}}.
Table~\ref{tbl:params} also indicates the values of the regularization
parameter $\lambda$ and Gaussian kernel parameter $\sigma$ in
\eqref{eq:gaussk} selected using the {\svmlight} solver~\citep{Joa99}
to maximize the classification accuracy on each validation set. (For
\texttt{MNIST-E} we use the same parameters as in \texttt{MNIST}.)
%\texttt{ADULT}, \texttt{MNIST},
%\texttt{CCAT} and \texttt{COVTYPE} datasets are downloaded from the
%UCI Repository\footnote{\url{http://archive.ics.uci.edu/ml/}}.

For the first five moderate-size tasks, we compare all of our
algorithms against four publicly available codes. Two of these are the
cutting-plane methods CPNY~\citep{JoaFY09} and CPSP~\citep{JoaYu09}
that are implemented in the version 3.0 of of {\svmperf}. Both search
for a solution as a linear combination of approximate basis functions,
where the approximation is based on Nystr\"om sampling (CPNY) or on
constructing optimal bases (CPSP).
The other two comparison codes are {\svmlight}~\citep{Joa99}, which
solves the dual SVM formulation via a succession of small subproblems,
and LASVM~\citep{BorAnt05}, which makes a single pass over the data,
selecting pairs of examples to optimize with the SMO algorithm. The
original SVM-Perf~\citep{Joa06} and OCAS~\citep{FraS08} are not
included in the comparison because they cannot handle nonlinear
kernels.  

For the final large-scale task with \texttt{MNIST-E} data set, we
compare our algorithms using feature mapping approximation --
{\assetonline} and {\assetstonline} -- to the online algorithm LASVM.

For our codes, the averaging parameter is set to $\bar{N}=m-100$ for
all experiments (that is, averaging is performed for the final $100$
iterates), and the error values are computed using the efficient
classification schemes of Section~\ref{sec:testing}.
% As this scheme uses approximate classifiers, the error values could
% be overestimating the real ones.

%\subsection{How does classification performance improve as we increase rank?} 
\subsection{Accuracy vs. Approximation Dimension}
\label{sec:rankperf}

\begin{figure*}[ht!] %[ht!]
\centering
\subfigure[\texttt{ADULT}]{\includegraphics[width=.335\textwidth, trim=0px 15px  8px 15px]{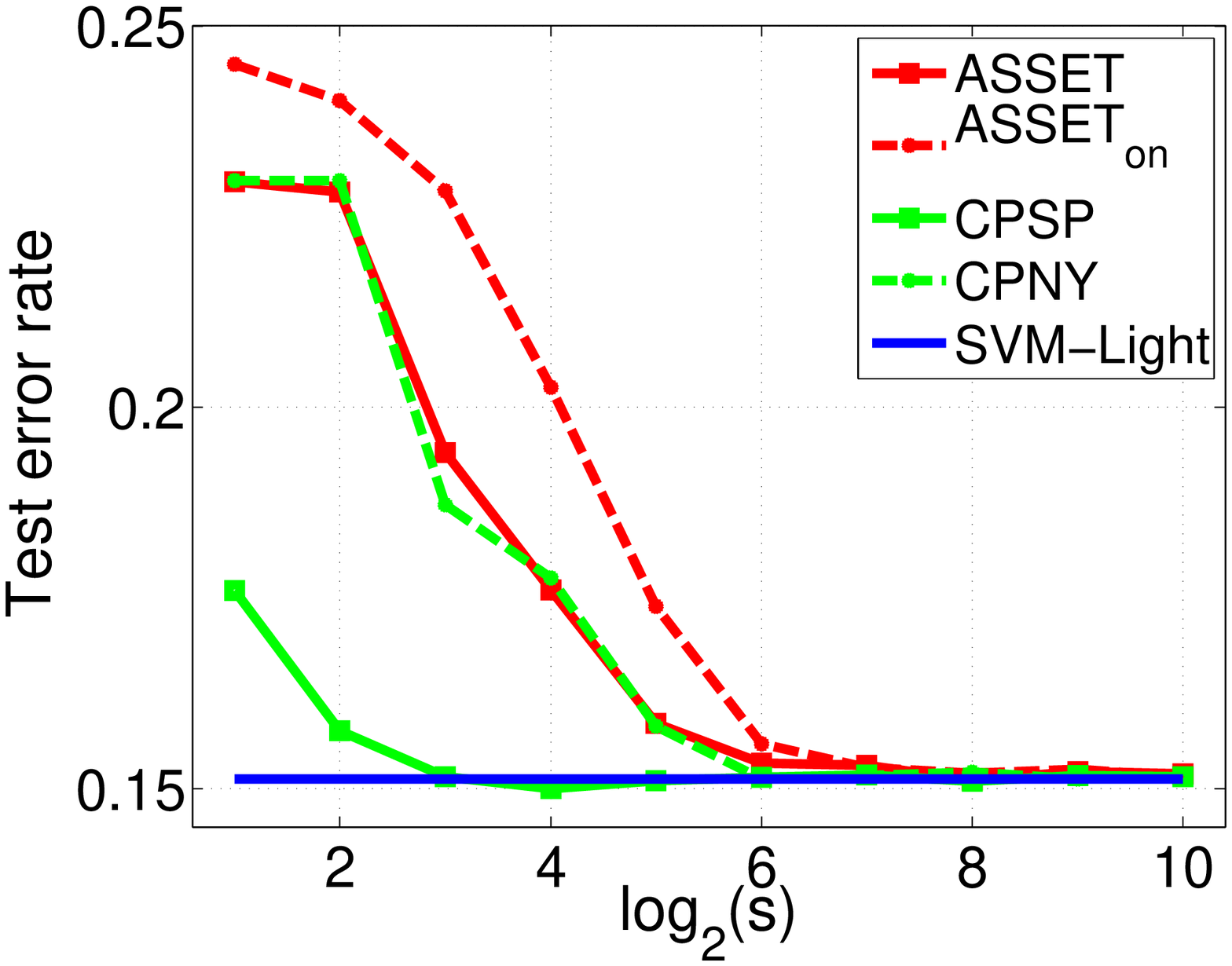}\label{fig:rank-adult}}
\subfigure[\texttt{MNIST}]{\includegraphics[width=.30\textwidth, trim=0px -16px 8px 15px]{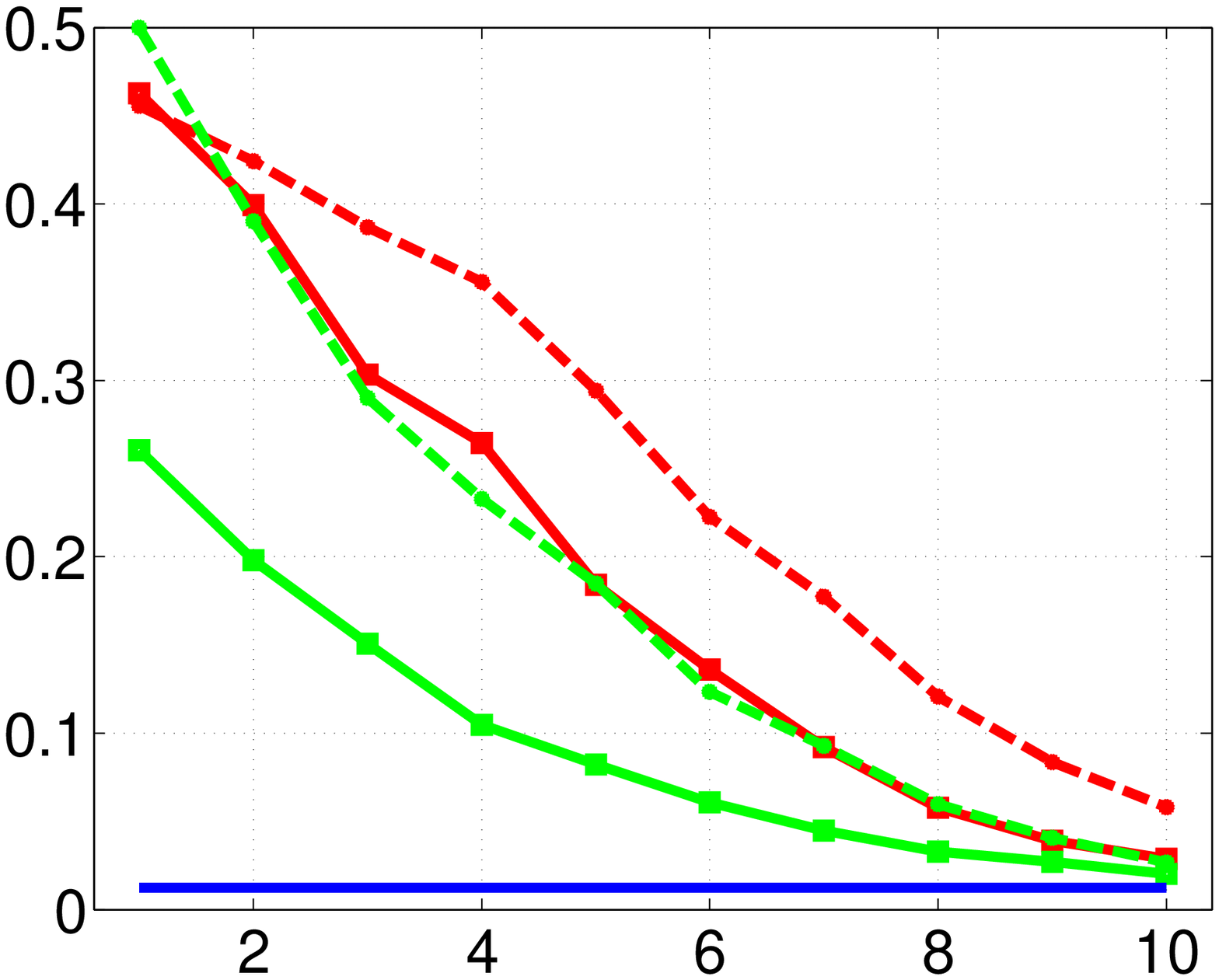}\label{fig:rank-mnist}}
\subfigure[\texttt{CCAT}]{\includegraphics[width=.30\textwidth, trim=0px -16px  8px 15px]{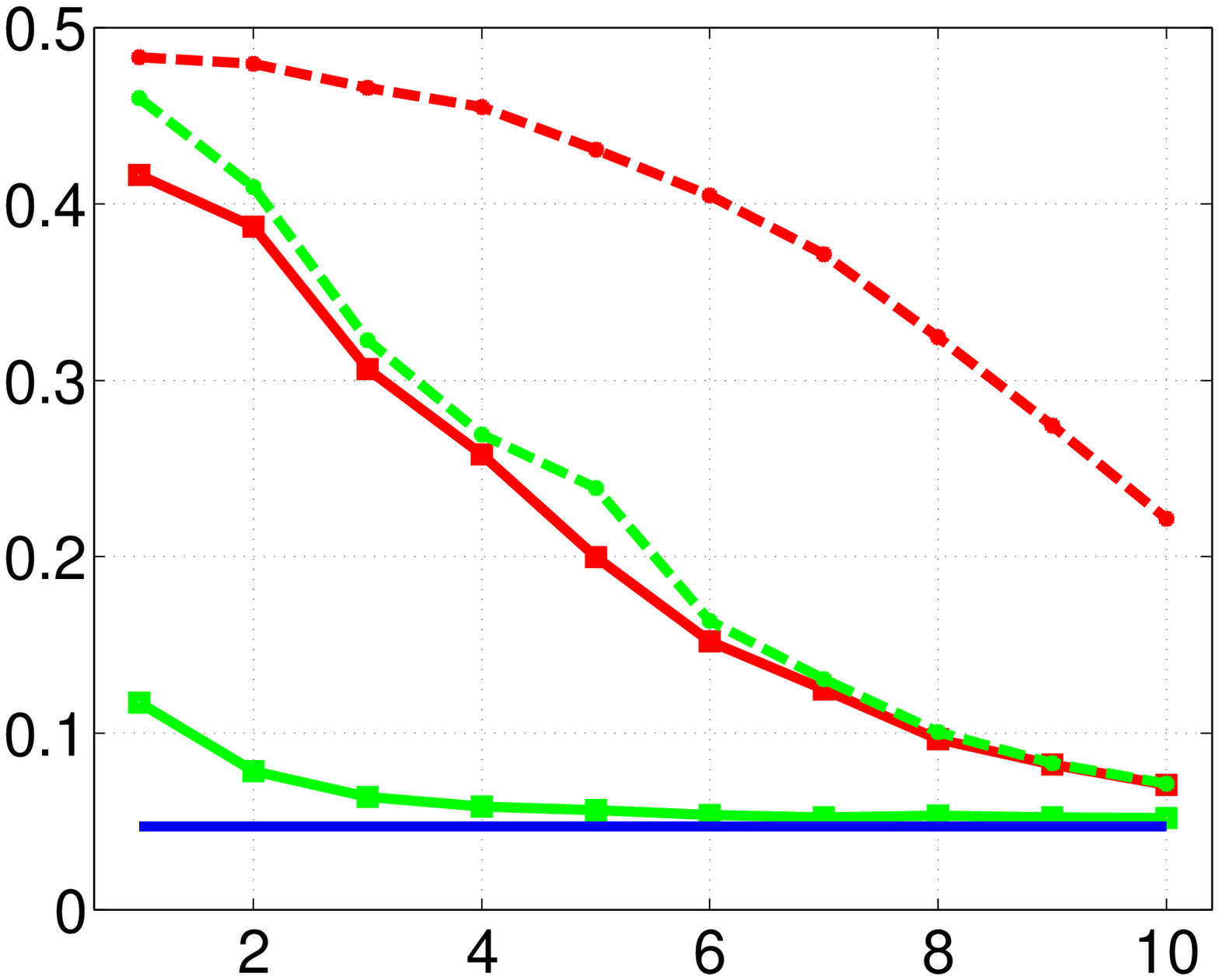}\label{fig:rank-ccat}}\\
\subfigure[\texttt{IJCNN}]{\includegraphics[width=.30\textwidth, trim=0px -15px 8px 15px]{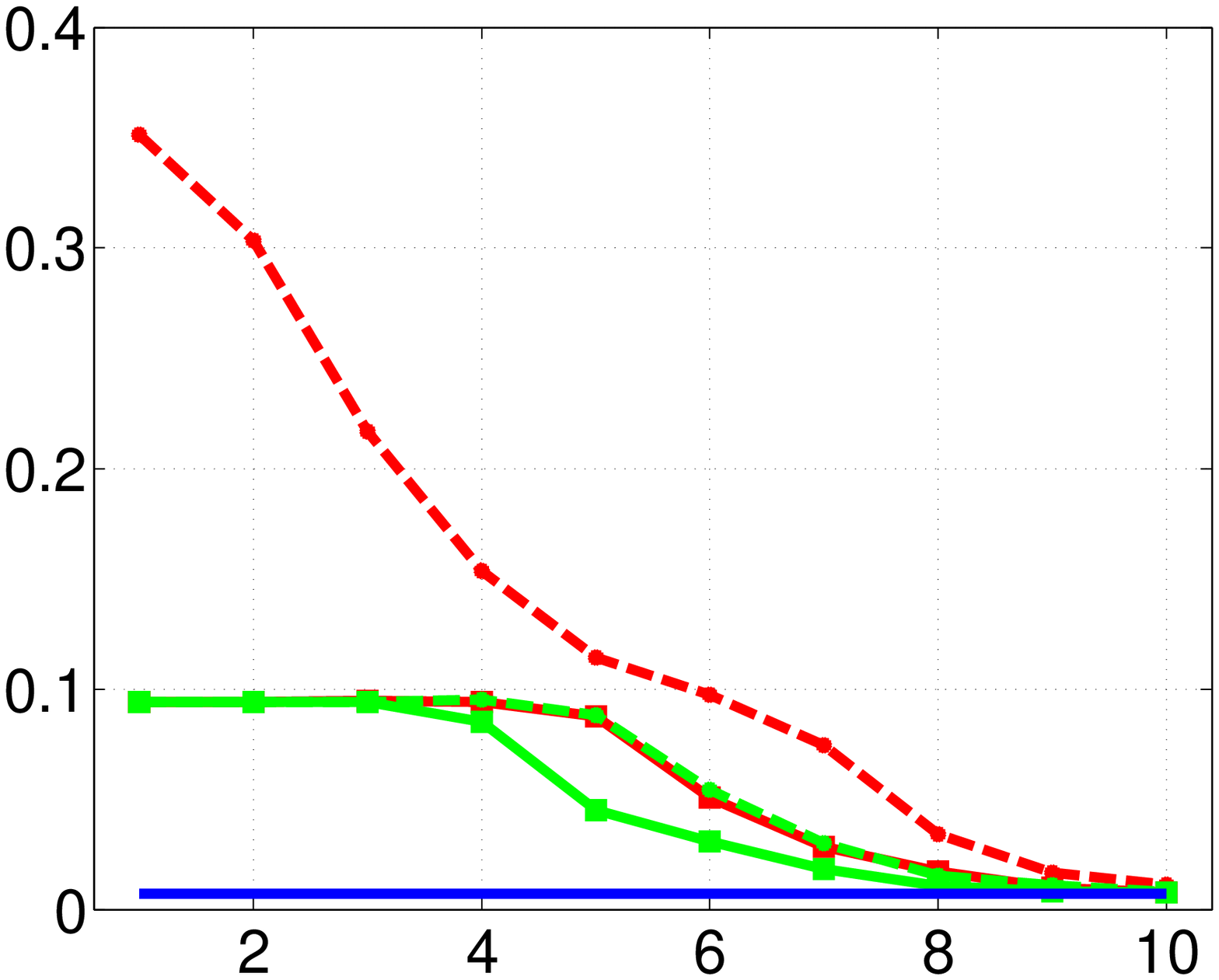}\label{fig:rank-ijcnn}}
\subfigure[\texttt{COVTYPE}]{\includegraphics[width=.30\textwidth, trim=0px -15px 8px 15px]{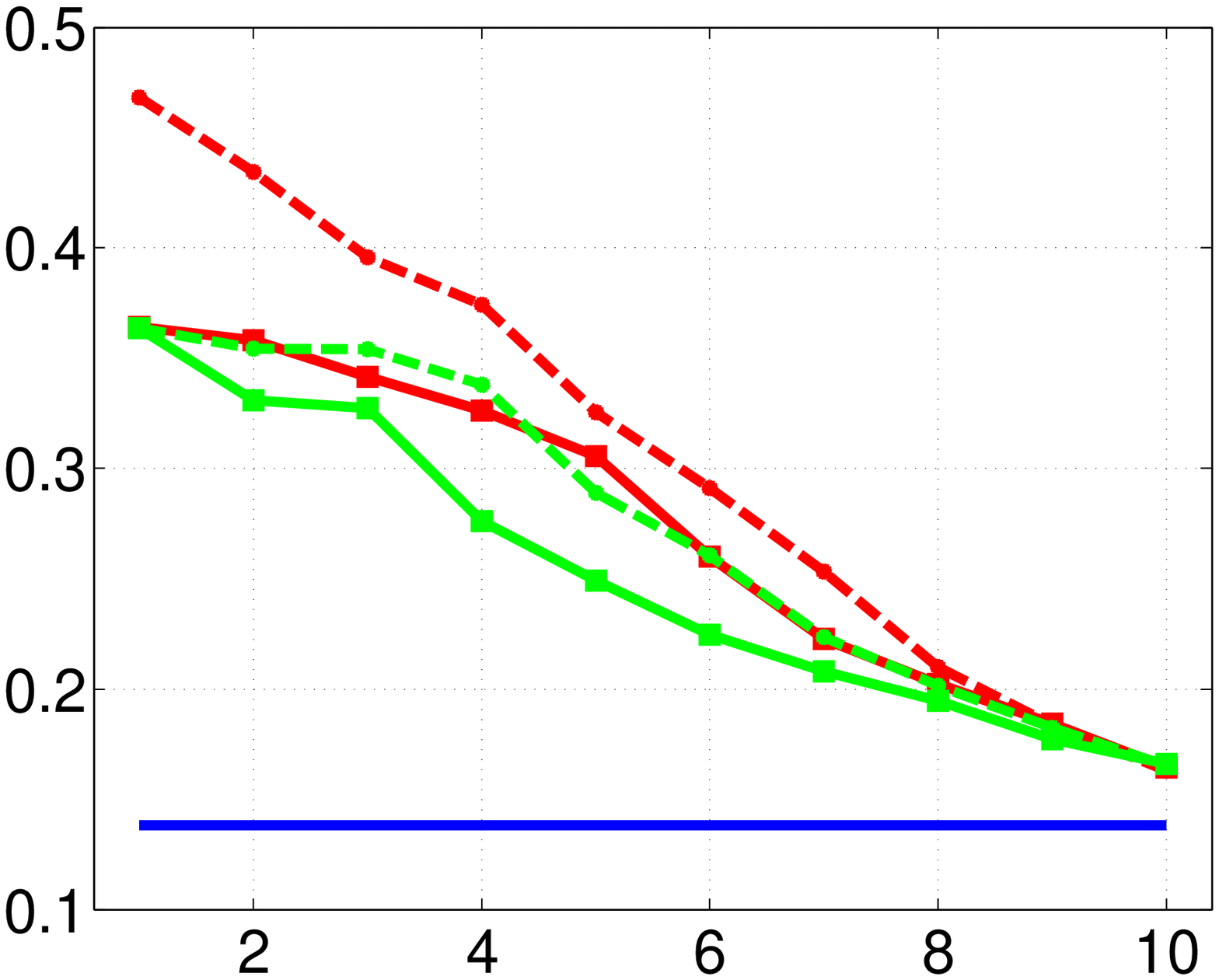}\label{fig:rank-covtype}}
%  \subfigure[\texttt{HUMAN}]{\includegraphics[scale=.29]{figures/human}\label{fig:rank-human}}
\caption{The effect of the approximation dimension to the test
  error. The x-axis shows the values of $s$ in log scale (base 2).
% For
%  {\assetonline}, $d=s$ and for the others $d\leq s$.
  \label{fig:rank-testerr}}
%  \hfill
\end{figure*}

The first experiment investigates the effect of kernel approximation
dimension on classification accuracy. We set the dimension parameter
$s$ in Section~\ref{sec:approxK} to values in the range $[2,1024]$,
with the eigenvalue threshold $\epsilon_d=10^{-16}$. Note that $s$ is
an upper bound on the actual dimension $d$ of approximation for
{\assetboth}, but is equal to $d$ in the case of {\assetbothon}.
The CPSP and CPNY have a parameter similar to $s$ (as an upper bound
of $d$); we compared by setting that parameter to the same values as
for $s$.  

For the first five moderate-size tasks, we ran our algorithms for 1000
epochs ($1000 m$ iterations) so that they converged to a near-optimal
value with small variation among different randomization.  We obtained
the baseline performance of these tasks by running {\svmlight}.
{\svmlight} does not have dimension parameters but can be expected to
give the best achievable performance by the kernel-approximate
algorithms as $s$ approaches $m$.

Figure~\ref{fig:rank-testerr} shows the results. Since {\asset} and
{\assetst} yield very similar results in all experiments, we do not
plot {\assetst}. (For the same reason we show only
{\assetonline} without {\assetstonline}.)
% We would expect the codes to perform
% well when the underlying kernel is well approximated by a
% low-dimensional surrogate.  \footnote{SJW: Could we use ``rank''
%   instead of ``dimension''? Or are both OK?}  SKL: I think we should
% keep ``dimension''.  The Gaussian kernel's eigen-spectrum depends on
% data as well as the parameter $\sigma$ in \eqref{eq:gaussk}.
When the value of $\sigma$
%in \eqref{eq:gaussk} 
is very small, as in Figure~\ref{fig:rank-adult} of \texttt{ADULT}
data set, all codes achieve good classification performance for small
dimension. In other data sets, the chosen values of $\sigma$ are
larger and the intrinsic rank of the kernel matrix is higher, so
classification performance continues to improve as $s$ increases.
% In particular, it is known that linear kernels work as well as
% nonlinear kernels on the \texttt{CCAT}.  If linear kernels are
% optimal for \texttt{CCAT}, the optimal Gaussian kernel may choose a
% very small value of $\sigma$ producing near-identity thus high-rank
% Gram matrix.
%
% Since feature vector dimension is large for this data set, the
% linear kernel matrix may be of full rank, so we can surmise that the
% Gaussian kernel may need to have high rank (that is, a smaller value
% of $\sigma$) in order to achieve good classification.

Interestingly, {\assetonline} (feature mapping approximation) seems to
require more dimension than {\asset} (kernel matrix approximation) to
produce similar classification accuracy.
%  suffer from
%approximating the kernel function rather than the kernel matrix; 
However in practice we can specify larger dimension for {\assetonline}
than for {\asset} since the former requires less computation than the
latter.
%
% The former is generally a more difficult problem.
For a given dimension, the overall performance of {\assetonline} is
worse than other methods, especially in the \texttt{CCAT} experiment.

The cutting plane method CPSP generally requires lower dimension than
the others to achieve the same prediction performance. 
% The reason would be the fact The power seems to come from the fact
% that
This is because CPSP spends extra time to construct optimal basis
functions, whereas the other methods depend on random
sampling. However, all approximate-kernel methods including CPSP
suffer considerably from the restriction in dimension for the
\texttt{COVTYPE} task.

\subsection{Speed of Achieving Similar Test Error}

\begin{table*}[t!]
  \caption{Training CPU time (in seconds, h:hours) and test error rate (\%) in parentheses. 
    Kernel approximation dimension is varied by setting $s=512$ and $s=1024$ for {\asset}, {\assetst}, CPSP 
    and CPNY. Decomposition methods do not depend on $s$, so their results are the same in both tables.\label{tbl:testerr}}
  \center
\begin{tabular}{|l||r@{\zg}r|r@{\zg}r||r@{\zg}r|r@{\zg}r||r@{\zg}r|r@{\zg}r|}\hline
  & \mc{4}{c||}{Subgradient Methods}    & \mc{4}{c||}{Cutting-plane} & \mc{4}{c|}{Decomposition} \\ \hline\hline
  $s=512$        & \mc{2}{c|}{\asset} & \mc{2}{c||}{\assetst} & \mc{2}{c|}{CPSP}  & \mc{2}{c||}{CPNY}& \mc{2}{c|}{LASVM}   & \mc{2}{c|}{\svmlight} \\ \hline
  \texttt{ADULT}  & {23} &(15.1$\pm$0.06) &  24 &(15.1$\pm$0.06)  & 3020 &(15.2)& 8.2h  &(15.1) & 1011  &(18.0) & 857 &(15.1)\\
  \texttt{MNIST}  & {97} &(4.0$\pm$0.05)  & 101 &(4.0$\pm$0.04)   & 550 &(2.7)  & 348   &(4.1)  & 588   &(1.4)  & 1323 &(1.2)\\
  \texttt{CCAT}   & 95 &(8.2$\pm$0.08)    & 99 &(8.3$\pm$0.06)    & 800 &(5.2)  & {62}  &(8.3)  & 2616  &(4.7)  & 3423 &(4.7)\\
  \texttt{IJCNN}  & {87} &(1.1$\pm$0.02)  & 89 &(1.1$\pm$0.02)    & 727 &(0.8)  & 320   &(1.1)  & 288   &(0.8)  & 1331 &(0.7)\\
  \texttt{COVTYPE}& 697 &(18.2$\pm$0.06)  & {586}&(18.2$\pm$0.07) & 1.8h &(17.7) & 1842 &(18.2) & 38.3h &(13.5) & 52.7h &(13.8)\\
%\texttt{HUMAN*} &        &         &                 &          &        &           &         &          &       &            &       &         \\
\hline\hline
%\end{tabular}
%  \caption{Training CPU time (s:seconds, h:hours) and test error in parentheses. 
%The number of samples $c$ for kernel approximation is fixed to 1024 for \assetst, ASSET, CPSP 
%and {\svmperf}. We set the eigenvalues cut-off threshold to $\epsilon_d=10^{-16}$
%\label{tbl:testerr}.}\smallskip
%\center
%\begin{tabular}{|l@{\zg}||r@{\zg}r@{\zg}|r@{\zg}r|r@{\zg}r|r@{\zg}r||r@{\zg}r|r@{\zg}r|}\hline
  $s=1024$        & \mc{2}{c|}{\asset} & \mc{2}{c||}{\assetst} & \mc{2}{c|}{CPSP}  & \mc{2}{c||}{CPNY}& \mc{2}{c|}{LASVM}   & \mc{2}{c|}{\svmlight} \\ \hline
 \texttt{ADULT}  &{78} &(15.1$\pm$0.05)  &83 &(15.1$\pm$0.04)   & 3399 &(15.2) & 7.5h &(15.2)  & 1011 &(18.0)  & 857 &(15.1)\\
 \texttt{MNIST}  &{275} &(2.7$\pm$0.03)  &275 &(2.7$\pm$0.02)   & 1273 &(2.0)  & 515  &(2.7)   & 588  &(1.4)   & 1323 &(1.2)\\
 \texttt{CCAT}   & 265  &(7.1$\pm$0.05)  &278 &(7.1$\pm$0.04)   & 2950 &(5.2)  & {123}&(7.2)   & 2616 &(4.7)   & 3423 &(4.7)\\
 \texttt{IJCNN}  &307   &(0.8$\pm$0.02)  &{297} &(0.8$\pm$0.01) & 1649  &(0.8) & 598  &(0.8)   & 288  &(0.8)   & 1331 &(0.7)\\
 \texttt{COVTYPE}&2259 &(16.5$\pm$0.04)  &{2064}&(16.5$\pm$0.06) & 4.1h &(16.6)& 3598 &(16.5)  & 38.3h &(13.5) & 52.7h &(13.8)\\
% \texttt{HUMAN*} &        &         &                 &          &       &           &         &          &       &            &       &         \\
\hline
\end{tabular}
\end{table*}

Here
%In performing runtime comparisons, 
we ran all algorithms other than ours with their default stopping
criteria. For {\asset} and {\assetst}, we checked the classification
error on the test sets ten times per epoch, terminating when the error
matched the performance of CPNY. (Since this code uses a similar
Nystr\"om approximation of the kernel, it is the one most directly
comparable with ours in terms of classification accuracy.) The test
error was measured using the iterate averaged over the 100 iterations
immediately preceding each checkpoint.

Results for the first five data sets are shown in
Table~\ref{tbl:testerr} for $s=512$ and $s=1024$. (Note that LASVM and
{\svmlight} do not depend on $s$ and so their results are the same in
both tables.)
%The shortest time values to achieve
%similar test accuracy are marked as bold, 
Our methods are the fastest in most cases.
% in most cases.
Although the best classification errors among the approximate codes
are obtained by CPSP, the runtimes of CPSP are considerably longer
than for our methods. In fact, if we compare the performance of
{\asset} with $s=1024$ and CPSP with $s=512$, {\asset} achieves
similar test accuracy to CPSP (except for \texttt{CCAT}) but is faster
by a factor between two and forty.  CPNY requires an abnormally long
run time on the \texttt{ADULT} data set; we surmise the code may be
affected by numerical difficulties.
% associated with the highly ill conditioned kernel for this problem.

It is noteworthy that {\asset} shows similar performance to {\assetst}
despite the less impressive theoretical convergence rate of the
former. This is because the values of optimal regularization parameter
$\lambda$ were near zero in our experiments, and thus the objective
function lost the strong convexity condition required for {\assetst}
to work. We observed similar slowdown of Pegasos and SGD when
$\lambda$ approaches zero for linear SVMs.

\subsection{Large-Scale Performance}

We take the final data set \texttt{MNIST-E} and compare the
performance of {\assetonline} and {\assetstonline} to the online SVM
code LASVM. (Other algorithms such as CPSP, CPNY, and {\svmlight} are
less suitable for large-scale comparison because they operate in batch
mode.)
% and thus require all training examples to
%be available at every stage of the run, while online algorithms see
%nly one example at a time.)
For a fair comparison, we fed the training samples to the algorithms
in the same order.

Figure~\ref{fig:mnist1m} shows the progress on a single run of our
algorithms, with various approximation dimensions $d$ (which is equal
to $s$ in this case) in the range $[1024,16384]$. Vertical bars
in the graphs indicate the completion of training. {\assetonline}
tends to converge faster and shows smaller test error values than
{\assetstonline}, despite the theoretical slower convergence rate of
the former. With $d=16384$, {\assetonline} and {\assetstonline}
required 7.2 hours to finish with a solution of $2.7\%$
% \footnote{SJW: the table says $3.0\%$?}  SKL: this is correct. The
% value in the table (3.0\%) is the averaged test error for the
% sampled ten 100-average iterates, whereas 2.7\% is with the final
% solution (the last 100-average iterate).
and $3.5\%$ test error rate, respectively. LASVM produced a better
solution with only $0.2\%$ test error rate, but it required 4.3 days
of computation to complete a single pass through the same training
data.
\begin{figure}[tb!] %[ht!]
\centering
\includegraphics[scale=.55,trim=20px 30px 2px 0px]{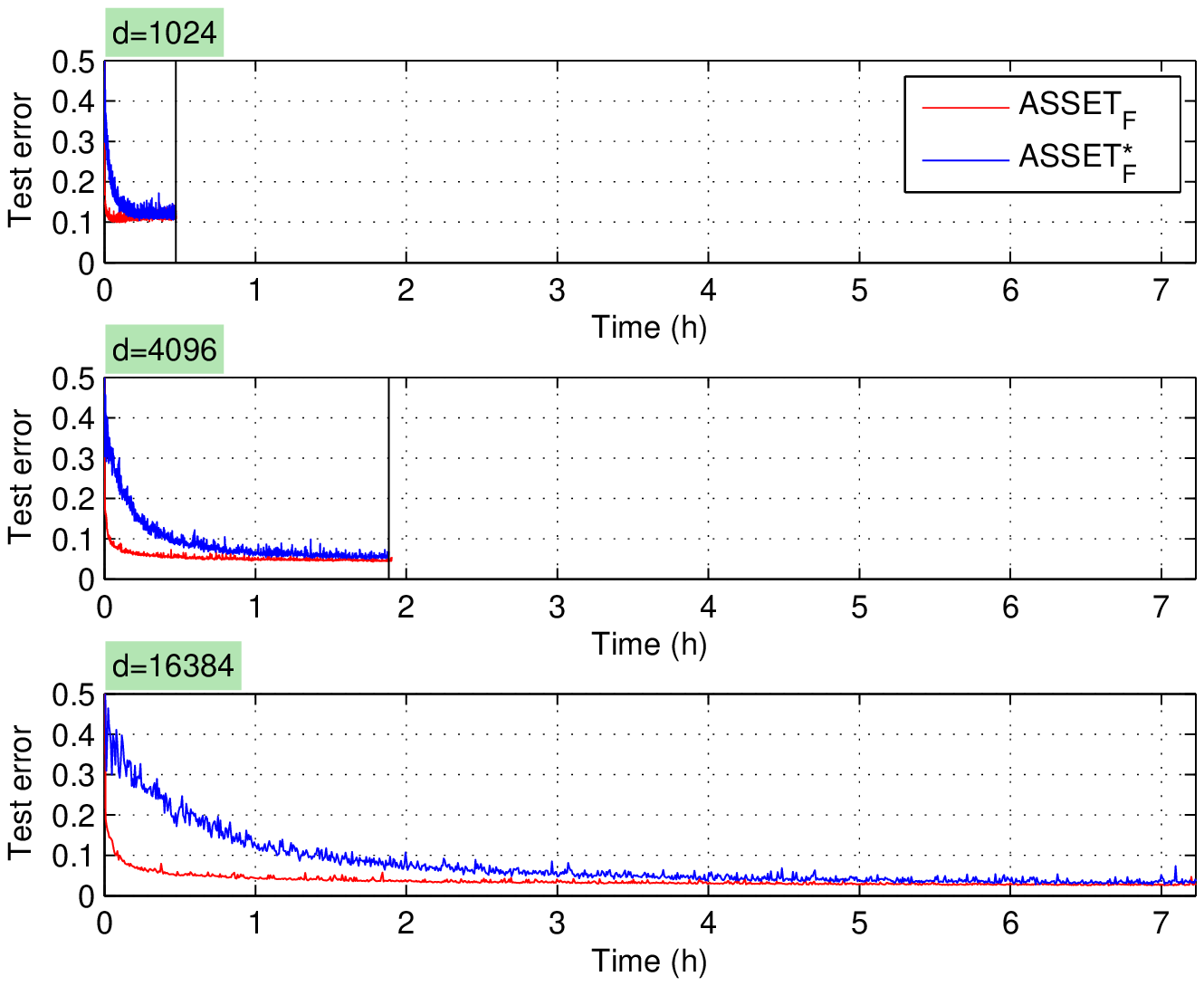}%\label{fig:mnist1m}
\smallskip
\caption{Progress of {\assetonline} and {\assetstonline} to their
  completion (\texttt{MNIST-E}), in terms of test error
  rate.\label{fig:mnist1m}}
\end{figure}

\section{{Conclusion}}
\noindent
We have proposed a stochastic gradient framework for training
large-scale and online SVMs using efficient approximations to
nonlinear kernels.
% support vector machines based on stochastic subgradients.  Our
% algorithms can operate in batch and in online mode, and allows for
% the use of nonlinear kernels via kernel approximation and
% reformulation of the primal form. They do not require strong
% convexity.
Since our approach does not require strong convexity of the objective
function or dual reformulations for kernelization, it can be extended
easily to other kernel-based learning problems.
% Despite its less appealing convergence, our method finds solutions
% of reasonable quality in short time.  Since our approach does not
% requiredoes not require strong convexity nly (weak) convexity of the
% objective function, they can be extended easily to regression,
% ranking, and other learning problems.

\section*{{Acknowledgements}}
\noindent
The authors acknowledge the support of NSF Grants DMS-0914524 and
DMS-0906818, and of the German Research Foundation (DFS) grant for the
Collaborative Research Center SFB~876: ``Providing Information by
Resource-Constrained Data Analysis''.

\bibliographystyle{abbrvnat}
%\bibliographystyle{apalike}
%{\small
\bibliography{asset_icpram} 
%}

\end{document}